\documentclass{article}

\PassOptionsToPackage{round}{natbib}

\usepackage[preprint]{neurips_2024}

\usepackage[utf8]{inputenc} 
\usepackage[T1]{fontenc}    
\usepackage{hyperref}       
\usepackage{url}            
\usepackage{booktabs}       
\usepackage{amsfonts}       
\usepackage{nicefrac}       
\usepackage{microtype}      
\usepackage{xcolor}         
\usepackage{graphicx}
\usepackage{dsfont}
\usepackage{color}
\usepackage{booktabs} 
\usepackage{subcaption}

\usepackage{algorithm}
\usepackage[noend]{algorithmic}

\usepackage{amsmath,amsfonts,bm}

\def\eqref#1{equation~\ref{#1}}

\def\1{\bm{1}}

\def\rmA{{\mathbf{A}}}

\def\rmQ{{\mathbf{Q}}}

\def\rmX{{\mathbf{X}}}

\DeclareMathAlphabet{\mathsfit}{\encodingdefault}{\sfdefault}{m}{sl}
\SetMathAlphabet{\mathsfit}{bold}{\encodingdefault}{\sfdefault}{bx}{n}

\def\gA{{\mathcal{A}}}

\def\gF{{\mathcal{F}}}

\def\gO{{\mathcal{O}}}

\def\gS{{\mathcal{S}}}
\def\gT{{\mathcal{T}}}

\def\sN{{\mathbb{N}}}

\def\sP{{\mathbb{P}}}

\def\sR{{\mathbb{R}}}

\newcommand{\E}{\mathbb{E}}

\DeclareMathOperator*{\argmax}{arg\,max}

\usepackage{amsthm}
\newtheorem{theorem}{Theorem}
\newtheorem{lemma}{Lemma}

\newtheorem{corollary}{Corollary}

\newtheorem{Exploring starts}[theorem]{Exploring starts}

\title{Finite-Sample Analysis of the Monte Carlo Exploring Starts Algorithm for Reinforcement Learning}

\author{%
  Suei-Wen Chen \\
  Department of Mathematics\\
  New York University Abu Dhabi\\
  Abu Dhabi, UAE \\
  \texttt{sueiwen.chen@nyu.edu} \\
  \And
  Keith Ross \\
  Department of Computer Science\\
  New York University Abu Dhabi \\
  Abu Dhabi, UAE \\
  \texttt{keithwross@nyu.edu} \\
  \And
  Pierre Youssef \\
  Department of Mathematics\\
  New York University Abu Dhabi \\
  Abu Dhabi, UAE \\
  \texttt{yp27@nyu.edu} \\
}

\begin{document}

\maketitle

\begin{abstract}
Monte Carlo Exploring Starts (MCES), which aims to learn the optimal policy using only sample returns, is a simple and natural algorithm in reinforcement learning which has been shown to converge under various conditions. However, the convergence rate analysis for MCES-style algorithms in the form of sample complexity has received very little attention.
In this paper we develop a finite sample bound for a modified MCES algorithm which solves the stochastic shortest path problem. To this end, we prove a novel result on the convergence rate of the policy iteration algorithm. This result implies that with probability at least $1-\delta$, the algorithm returns an optimal policy after $\tilde{\gO}(SAK^3\log^3\frac{1}{\delta})$ sampled episodes, where $S$ and $A$ denote the number of states and actions respectively, $K$ is a proxy for episode length, and $\tilde{\gO}$ hides logarithmic factors and constants depending on the rewards of the environment that are assumed to be known.

\end{abstract}

\section{Introduction}
Some of the most successful reinforcement learning (RL) algorithms are based on Monte Carlo (MC) methods. For example, in the outer loop of training, AlphaZero \citep{silver2018general} contains a Monte Carlo component that runs episodes to completion and uses the returns from those episodes for the targets in the loss function. As another example, the policy gradient algorithm REINFORCE \citep{williams1992simple}, one of the most fundamental algorithms in deep reinforcement learning, also uses Monte Carlo returns to update network parameters. 

Arguably, Monte Carlo Exploring Starts (MCES) is the most basic and natural algorithm for episodic MDPs in tabular RL. 
In MCES, after a (random-length) episode, the estimate for the Q-function (also called the action-value function) is updated with the Monte Carlo return for each state-action pair along the episode, and the policy is then updated as greedy with respect to the current Q-function estimate. The need for exploration is satisfied by the exploring starts assumption, which means that every state-action pair must be chosen to be the start of an episode infinitely often.
In the 2018 classic book, \cite{sutton2018reinforcement} write that the convergence of MCES to the optimal value function is ``one of the most fundamental open theoretical questions in reinforcement learning''. 

Although significant progress has been made in recent years in understanding the convergence of MCES \citep{tsitsiklis2002convergence,chen2018convergence,liu2020convergence,wang2020convergence},  
relatively less progress has been made with rates of convergence. 
While there are guarantees for Monte Carlo algorithms, such as regret bounds or sample complexity, most existing results concern either discounted infinite-horizon tasks or fixed-length finite-horizon tasks. 
Little attention has been paid, however, to episodic MDPs in which the lengths of episodes are random and the returns are undiscounted, even though this is arguably the most natural setting in many applications. For the known convergent variants of MCES algorithms that apply to undiscounted episodic MDPs \citep{tsitsiklis2002convergence, chen2018convergence, liu2020convergence},
the proofs rely on stochastic approximation arguments which do not allow for a straightforward convergence rate analysis.

In this paper we address the sample complexity for learning undiscounted episodic MDPs with MCES algorithm. Instead of using the stochastic approximation machinery as in previous related works, our strategy is to reduce the frequency at which the policy is updated to obtain more accurate Q-function estimates that lead to better policy improvement steps. The contribution of this work is twofold. We first provide an upper bound on the number of steps which the policy iteration algorithm \citep{sutton2018reinforcement} needs to take before reaching optimality (Theorem \ref{theorem:policy-iteration-steps}). 
Based on this analysis, we then propose an MCES-style algorithm which outputs an optimal policy with probability $1-\delta$ after sampling $\tilde{\gO}\left(SAK^3\log^3\frac{1}{\delta}\right)$ trajectories. Here, $S$ is the number of non-terminal states, $A$ is the number of actions, $K$ is a proxy for episode lengths, and $\tilde{\gO}$ hides the logarithmic terms and constants related to the reward parameters of the MDP. The exact dependency on these constants is presented in Corollary \ref{corollary:main}. To the best of our knowledge, this work is the first to present a finite sample analysis for learning an exact optimal policy (instead of $\epsilon$-optimal policies) in undiscounted episodic MDPs using Monte Carlo methods.

\section{Related Work} \label{section:related-work}

\citet{bertsekas1996neuro} gave a counterexample (in Example 5.12) of a continuing MDP for which the MCES fails to converge. However, their setting differs from the original MCES algorithm in \cite{sutton2018reinforcement} which concerns episodic tasks. 
\cite{wang2020convergence} constructed a counterexample of a three-state episodic MDP for which MCES diverges almost surely. Thus the MCES algorithm originally proposed in \cite{sutton2018reinforcement} is not guaranteed to converge. To ensure convergence, either the orginal MCES algorithm needs to be modified or restrictive assumptions need to be imposed on the underlying MDP. 

\cite{tsitsiklis2002convergence} made significant progress in establishing convergence for a modified version of MCES using a stochastic approximation argument. Specifically, Tsitsiklis showed that MCES converges almost surely if 1) all  state-action pairs are chosen at the start of episodes with the same frequency, 2) the Q-function estimate is only updated for the initial state-action pair in each episode, and 3) the discount factor is strictly less than one. Employing the same general methodology as in \cite{tsitsiklis2002convergence}, \cite{chen2018convergence} extended this result to the undiscounted case where all policies are assumed to be proper. \cite{liu2020convergence} further relaxed the assumption that all policies are proper and also explored the effect of component-dependent learning rate.
Using an induction argument, \citet{wang2020convergence} proved that without any of the three assumptions in \cite{tsitsiklis2002convergence}, the original MCES algorithm converges almost surely for a subclass of Markov decision processes (MDPs) in which no state is revisited under any optimal policy.

The convergence rate of reinforcement learning is an important part of theoretical guarantees that quantify the efficiency of an algorithm. One particular notion of convergence is the \emph{probably approximately correct} (PAC) framework, in which an algorithm returns an $\epsilon$-optimal policy with probability at least $1-\delta$ for some accuracy parameter $\epsilon\geq 0$ and confidence parameter $\delta \in [0,1)$.
Although this form of convergence has been heavily studied in the infinite-horizon discounted setting as well as the finite-horizon setting in which every episode has a predefined fixed length, little progress has been made for the undiscounted episodic case.

An extensive body of work on the sample complexity of discounted environments has been established \citep[see, for example,][]{ kearns1999finite, strehl2006pac,  strehl2009reinforcement,  lattimore2012pac, azar2013minimax, sidford2018near, yang2019sample, li2024breaking}. On the other hand, the analysis for finite-horizon MDPs did not receive much attention until fairly recently. 
\citet{dann2015sample} were among the first to give upper and lower PAC bounds in the finite horizon setting; 
\citet{domingues2021episodic} gave lower bounds for the PAC complexity by constructing a class of hard MDPs;
\citet{li2022settling} studied the dependence of sample complexity on the horizon length;
\citet{al2023towards} gave instance-dependent PAC bounds for MDPs with Gaussian rewards and studied exact best policy identification in certain settings; \citet{tirinzoni2023optimistic} also gave instance-dependent PAC bounds using some suboptimality gap measures and minimal visitation probabilities. 

However, many tasks are more naturally modelled as undiscounted tasks with random episode lengths, such as board games and robotics applications.
Such an episodic setting is equivalent to the stochastic shortest path problem (SSPP), which is a strictly more general setting than the discounted and finite-horizon settings \citep{bertsekas1995dynamicVol1, bertsekas1995dynamicVol2} and is also strictly harder to learn \citep{chen2023reaching}. As \citet{chen2021implicit} pointed out, the analysis for the SSPP is more challenging due to the lack of similar structures as in the discounted or finite-horizon setting. Specifically, in SSPP the return is the undiscounted sum of rewards, and the episode length is no longer fixed and is typically unbounded. 
If we let $\gamma \in (0,1)$ be the discount factor in the discounted setting and $H$ be the trajectory length in the finite-horizon setting, the sample complexity bounds involve horizon parameters  of the form $(1-\gamma)^{-1}$ or $H$. A natural analogue for the effective horizon in the SSPP is the SSP-diameter \citep{tarbouriech2020no}, defined by the minimum expected episode length over all policies when maximized over all starting states. In this work, we consider another quantity ($w$ as defined in Section \ref{subsection:contraction}) derived from the expected trajectory length. This quantity defines the underlying contraction structure of the SSPP when all policies are proper and governs the tail of the distribution of trajectory lengths, which we exploit to bound the estimation error of the Q-function.

While there have been a number of results studying the regret bounds for the SSPP \citep{tarbouriech2021stochastic, pmlr-v119-rosenberg20a, tarbouriech2020no, chen2021implicit, chen2022policy}, unlike the fixed-horizon setting it is not possible to convert regret bounds in SSPP into PAC bounds \citep{tarbouriech2021sample} and therefore direct analysis is required to obtain PAC guarantees. \citet{tarbouriech2021sample} were the first to 
develop PAC bounds for a model-based Monte Carlo algorithm
by maintaining confidence bounds for the transition probabilities. The complexity bound depends on the worst-case upper bound $T_\ddagger$ on the expected hitting time $T_\star$ of the optimal policy, defined as the maximum of the optimal value function $\|v_\star\|_\infty$ divided by the minimum immediate cost $c_{\min}$. \citet{chen2023reaching} showed that this dependence is unavoidable and that it is in general not possible to learn a SSPP without access to a generative model or prior knowledge on the environment. Under certain conditions, \cite{chen2023reaching} proposed the first model-free algorithm \texttt{BPI-SSP} whose sample complexity has a minimax-optimal dominating term (Theorem 11). This seems to be the only existing PAC analysis on model-free approach to solving SSPP, but they imposed the strong assumption that there exists a special terminating action common to all states. As such, the class of MDP to which their result applies is highly restrictive.

In this paper, we develop a sample complexity bound for an MCES-style algorithm that is model-free and approximates only the Q-function. Unlike the strategy adopted by \citet{chen2023reaching}, we exploit the underlying contraction property of the MDP when all policies are proper and directly give an analysis of generalized policy iteration. Based on this analysis, we propose an algorithm that outputs an optimal policy with high probability, namely a PAC sample bound with $\epsilon=0$. 
Assuming access to some form of suboptimality gap and termination probabilities, we derive a sample complexity bound in parameters that have not been considered in previous work. In particular, while we assume more knowledge on the structure of the MDP compared to \cite{chen2023reaching}, our algorithm applies to a different class of MDPs and identifies an \emph{exact} optimal policy instead of just an $\epsilon$-optimal policy for $\epsilon>0$, and we also provide an exact number of samples which guarantees optimality (Theorem \ref{theorem:main}).

\section{Preliminaries} \label{section:setup}
\subsection{Markov Decision Processes (MDPs)}

A finite Markov decision process is defined by a finite state space $\gS$ and a finite action space $\gA$, reward function $r(s,a)$ mapping $\gS \times \gA$ to the reals, and transition kernels $p(\cdot|s,a)$ for every state $s \in \gS$ and action $a \in \gA$ which give the distribution of the next state. A (deterministic and stationary) \emph{policy} $\pi$ is a mapping from the state space $\gS$ to the action space $\gA$. We denote by $\pi(s)$ the action selected under policy $\pi$ when in state $s$. Denote by $s_t^{\pi}$ the state at time $t$ under policy  $\pi$.  Given any policy $\pi$, the state evolution becomes a well-defined Markov chain with transition probabilities
\[
\sP(s_{t+1}^{\pi} = s'|s_{t}^{\pi} = s) := p(s'|s,\pi(s)).
\]

We say that an MDP is \emph{episodic} if the state space $\gS=\gS'\cup \{\tilde{s}\}$ consists of transient states $\gS'$ and an absorbing state $\tilde{s}$ which satisfies $p(\tilde{s}|\tilde{s}, a)=1$ for all actions $a$. (If there are multiple terminal states, without loss in generality they can be lumped into one state.)
A policy is said to be \emph{proper} if, starting from any state, a trajectory following the policy will reach the terminal state with probability $1$ \citep{bertsekas1995dynamicVol2}.  In other words, a policy is proper if the Markov chain induced by the policy is absorbing with a single recurrent class $\{\tilde{s}\}$. Example of episodic MDPs in which all policies are proper include card and board games such as  ``Blackjack'' \citep{sutton2018reinforcement}, Chess and Go \citep{silver2018general}. We denote by $S:=|\gS'|$ and $A:=|\gA|$ the number of transient states and actions respectively.

\subsection{Value Functions}
Let $T^{\pi} _s= \min \{ t  :  s_t^{\pi} = \tilde{s} \}$ be the length of an episode
\[
s_0 = s, a_0 = \pi(s_0), s_1, a_1=\pi(s_1), \dots, s_{T^\pi_{s}} = \tilde{s}
\]
under policy $\pi$ starting from $s\in \gS'$. The expected return at $s$ following $\pi$ defines the value function
\begin{equation}
v_\pi(s) := \E \left[ \sum_{t=0}^{T^{\pi}_s-1}  \gamma^t r(s_t^{\pi},\pi(s_t^{\pi})) \hspace{2pt} \Bigg| \hspace{2pt} s_0^{\pi} = s \right],\label{v-function}
\end{equation}

where $\gamma\in(0,1]$ is  the discount factor.
Similarly, for any state-action pair $(s,a)$, let $T^\pi_{s,a}$ be the length of a random episode which starts from state $s$, takes action $a$ and then follows $\pi$ thereafter:
\[
s_0 = s, a_0 = a, s_1, a_1 = \pi(s_1), \dots, s_{T^\pi_{s,a}-1}, a_{T^\pi_{s,a}-1} = \pi(s_{T^\pi_{s,a}-1}), s_{T^\pi_{s,a}} = \tilde{s}.
\]
The expected return when starting in state $s$ and action $a$ defines the action-value function
\begin{equation}
q_\pi(s,a) = \E \left[ \sum_{t=0}^{T^{\pi}_{s,a}-1}  \gamma^t r(s_t, a_t) \hspace{2pt} \Bigg| \hspace{2pt}  s_0= s , a_0=a\right]. \label{q-function}
\end{equation}

When $\gamma<1$, it is standard that there exists an optimal policy $\pi^\star$ satisfying
\begin{align}
    v_\star(s) & := \max_{\pi} v_\pi(s) = v_{\pi^\star}(s) , \label{eq:v_star} \\
    q_\star(s,a) & := \max_{\pi} q_\pi(s, a) = q_{\pi^\star}(s, a), \quad \text{and} \label{eq:q_star} \\
    v_\star(s) &= \max_{a\in\gA}q_\star(s,a). \label{eq:v-star-q-star}
\end{align}
When  $\gamma=1$, maximizing (\ref{v-function}) and (\ref{q-function}) corresponds to the stochastic shortest path problem (SSPP) \citep{bertsekas1991analysis, bertsekas1995dynamicVol2}. The difference is that in the formulation of the SSPP, the objective is to minimize the accumulated costs instead of maximizing the accumulated rewards.
Under some general assumptions, there exists a deterministic stationary policy $\pi^\star$ which is proper and satisfies (\ref{eq:v_star}), (\ref{eq:q_star}), and (\ref{eq:v-star-q-star}). In particular, this is true when all policies are proper, which further implies that $v_\pi$ and $q_\pi$ are finite with probability $1$ since the episode length is almost surely finite. We denote the collection of all optimal policies by $\Pi_\star$.

\subsection{Contraction Structure of  Episodic MDPs} \label{subsection:contraction}
Given an episodic MDP, let the vector $w\in \mathbb{R}^{S}$ be defined as the optimal solution to a modified MDP that has the same transition structures as the original one but whose reward for each transition is deterministically $1$. In other words, $w(i)$ is the maximum expected path length from the non-terminal state $i$ to $\tilde{s}$ under any policy. The vector $w$ has positive entries and induces a weighted sup-norm $\|\cdot\|_w$ on $\mathbb{R}^S$ defined by
    \[ \|v\|_w := \max_{i\in \gS'}\frac{|v(i)|}{w(i)}.\]
    
From Section 3.3 in \citet{bertsekas1995dynamicVol2}, when all policies are proper, the Bellman operator $\gT$ on $\mathbb{R}^S$ defined by 
\begin{align} \label{bellman-contraction}
(\gT v)(s):= \max_{a\in \gA} \left(r(s,a) + \sum_{s'\in\gS'}p(s'|s,a) v(s')\right)
\end{align}
is a contraction under $\|\cdot\|_w$ with constant $\rho:=1-\|w\|_\infty^{-1}$. However, in the presence of an improper policy, $\gT$ may not be a contraction with respect to any norm \citep[][p. 210]{bertsekas1995dynamicVol2}.

\subsection{Comparison between different MDP settings}

In this work we focus on \emph{episodic} MDPs, by which we mean MDPs where episodes terminate once the terminal state $\tilde{s}$ is reached. As such, the lengths of episodes can be random. This is distinct from \emph{finite-horizon} MDPs, in which all episodes share a common length $H$ called the horizon. For finite horizon MDPs, the optimal Q-function and optimal policy are in general non-stationary and depend on the time parameter. Some authors refer to finite horizon MDPs as episodic MDPs, but here we follow \citet{sutton2018reinforcement} and reserve the term \emph{episodic MDPs} for the more general setting where episode lengths can be random. We remark that a finite horizon MDP can always be converted to an episodic MDP (by adding the time parameter to the state) but not vice versa. 

We also point out that for continuing MDPs where $\gamma<1$, the effective horizon is $(1-\gamma)^{-1}$. This quantity appears in the performance bounds and plays an analogous role to the parameter $H$ in the fixed-horizon setting. In the episodic case, if we  view $\rho$ as the proxy for $\gamma$ on the basis that both $\rho$ and $\gamma$ are the discount factor of their respective Bellman operator, then the  effective horizon becomes $(1-\rho)^{-1}=\|w\|_\infty$, which is the maximum expected hitting time of the terminal state. In the analysis, we upper bound $\|w\|_\infty$  using the number of steps required so that the probability of reaching the terminal state is above a certain threshold (Lemma \ref{lemma:uniform-bound-on-transition-matrix}). 

\subsection{Suboptimality gaps}
Our sample complexity bound relies on separation parameters that measure the hardness of the MDP. Let $\mathcal{A}(\pi,s):=\argmax_{a\in\mathcal{A}} q_\pi(s,a)$
denote the set of maximizing actions at state $s$ following policy $\pi$. We define the \emph{suboptimality gap} of the policy $\pi$ by the minimum difference in action-value function between the best action the second-best action across all states:
\begin{align*}
    \Delta(\pi) :=  \min_s \left[\max_{a'\in \gA} q_\pi(s,a')- \max_{a\notin \mathcal{A}(\pi,s)}  q_\pi(s,a)\right] .
\end{align*}
For any optimal policy $\pi_\star$, we denote by $\Delta_\star:=\Delta(\pi_\star)$ as its suboptimality gap, which is well-defined because of Equations (\ref{eq:v_star}) and (\ref{eq:q_star}).
Finally, let $\Delta:=\min_\pi \Delta(\pi)$
denote the minimum suboptimality gap among all deterministic stationary policies. These quantities are strictly positive (unless the return is constant) and they measure the statistical difficulty in learning the optimal policy in the environment, which means that sample complexity bounds must depend on them. In Section \ref{section:main-results} we assume knowledge of the parameters $\Delta, \Delta_\star$ and  give a sample complexity bound in terms of them.

\section{Finite Sample Analysis of MCES Variants} \label{section:main-results}
In this section we present a finite sample analysis of learning an undiscounted episodic MDP using MCES. We assume without loss of generality that the immediate reward function $r$ is bounded in $[0,1]$. We also make the assumption that all policies are proper.

\subsection{Main Results}
A major hurdle for MCES to converge is the asynchronous updates which lead to the cycling behavior of the Q-function estimate.
To obtain a convergent MCES algorithm we turn to synchronous updates for each state-action pair and  decrease the frequency at which the policy is updated, as presented in Algorithm \ref{alg:mces-modified}. We refer to the first inner for loop as the \emph{policy evaluation step} and the second one as the \emph{policy improvement step}. 
Between each policy improvement step, we sample sufficiently many trajectories for every state-action pair to ensure accurate policy evaluation. As such, Algorithm \ref{alg:mces-modified} is a modification of the optimistic policy iteration algorithm \citep{tsitsiklis2002convergence} which enables a PAC sample complexity analysis.

\begin{algorithm*}[htbp]
\caption{MCES for Convergence with High Probability} 
	\label{alg:mces-modified}
	\begin{algorithmic}[1]
        \STATE \textbf{Input:}  $L\in\mathbb{N}$, $N\in\mathbb{N}$
	\STATE Initialize: $\pi(s) \in \gA$, $Q(s,a) \in \sR$, for all $s\in \gS', a \in \gA$ (arbitrarily) 
	\FOR{$t = 0,..., L-1$} 
	\FOR{each state action pair $(S_0, A_0) \in \gS'\times \gA$ }
    \STATE Generate $N$ episode $\{\tau_i\}_{i\in[N]}$ starting from $S_0, A_0$ following $\pi$ and observe returns $\{R_i\}_{i\in[N]}$ 
    \STATE $Q(S_0, A_0) \leftarrow \frac{1}{N}\sum_{i=1}^N R_i$ 
    \label{line:q-update}
        \ENDFOR
        \FOR{ $S\in \gS'$}
    \STATE $\pi(S) \leftarrow \argmax_a Q(S, a)$ (with ties broken in an arbitrary manner) 
    \label{line:policy-update}
        \ENDFOR
    \ENDFOR
    \STATE \textbf{Output:} $\pi$ 
    \end{algorithmic}
\end{algorithm*}

To ensure that Algorithm \ref{alg:mces-modified} converges with high probability, we would need to choose the number of iterations $L$ and the number of sampled trajectories per iteration $N$ sufficiently large. For $\eta\in(0,1)$ let $K_\eta>0$ be such that under any policy,  an episode starting from any state terminates in $K_\eta$ steps with probability at least $\eta$. That is, $K_{\eta}$ is any upper bound of the quantity $\overline{K}_\eta$ defined by
\begin{equation*}
    \overline{K}_{\eta} := \min \{ k \in\sN : \min_{\pi} \min_{s\in \gS'} P(T^{\pi}_s \leq k )\geq \eta  \} .
\end{equation*}
In Lemma \ref{lemma:uniform-bound-on-transition-matrix} we show that $K_\eta$ can be taken to be finite. Define 
\begin{equation} \label{L-num-iter}
    L(\eta):=\left\lceil \frac{2 K_\eta}{\eta} \log \frac{K_\eta}{\eta\Delta_\star }\right\rceil
\end{equation}
and for  $\eta = 1-e^{-1}$ we write $L_\star :=L(1-e^{-1})$ and $K:=K_{1-e^{-1}}$. For $\delta>0$ define
\begin{equation} \label{n-num-samples}
    N(\delta) := \left\lceil \frac{8K^2}{\Delta^2} \left( \log \frac{2SA}{1-(1-\delta)^{1/L_\star}}\right)^3 \right\rceil,
\end{equation}
where $S=|\gS'|$ is the number of transient states and $A=|\gA|$ is the number of actions. Theorem \ref{theorem:main} establishes that $L=L_\star$ iterations and $N=N(\delta)$ episodes per iteration are enough to make sure that the algorithm outputs an optimal policy with high probability.
\begin{theorem} \label{theorem:main}
    For any $\delta\in (0,1)$, with $L=L_\star$ improvement steps and $N=N(\delta)$ episodes per state-action pair between improvement steps, the resulting policy $\pi$ given by Algorithm \ref{alg:mces-modified} satisfies
    \begin{align*}
        \sP  (\pi \text{ is an optimal policy}) \geq 1-\delta.
    \end{align*}
\end{theorem}
\begin{corollary} \label{corollary:main}
    In the setting of Theorem \ref{theorem:main}, the total number of episodes sampled by Algorithm \ref{alg:mces-modified} is $SAL_\star N(\delta)$, which is of order
    \begin{equation}
        \gO\left(\frac{SA K^3}{\Delta^2} \left(\log \frac{K}{\Delta_\star}\right) \left(\log \frac{SA K}{\delta}\right)^3 \right).
    \end{equation}
\end{corollary}
We remark that even though Theorem \ref{theorem:main} is stated for $\eta=1-e^{-1}$, it remains valid for other choices of $\eta$ and their corresponding $L=L(\eta)$ if small adjustments on the constant factors are made. Moreover, Theorem \ref{theorem:main} remains valid if Algorithm \ref{alg:mces-modified} employs first-visit updates for the Q-function instead of only updating value for the starting state-action pair. This is because the first-visit returns are independent across different state-action pairs, and more trajectories can only increase statistical accuracy.

In order to establish Theorem \ref{theorem:main} we introduce several lemmas.
Lemmas \ref{lemma:uniform-bound-on-transition-matrix} and \ref{lemma:subexp-bound-on-absorption-time} ensure that the probability of an episode being too long is exponentially small. Lemma \ref{lemma:sampling-error-bound} gives the number of episodes needed for each state-action pair for accurate policy evaluation.  The proofs of these Lemmas can be found in Appendix \ref{appendix:proofs}.

\begin{lemma} \label{lemma:uniform-bound-on-transition-matrix}
For any $\eta \in (0,1)$ there exists some $K_\eta\in \mathbb{N}$ such that the transition matrix $\rmQ_\pi$  on the non-terminal states $\gS'$ induced by policy $\pi$ satisfies $\|\rmQ_\pi^k\mathbf{1}\|_\infty \leq 1-\eta$ for all $k\geq K_\eta$ and for all policies $\pi$.
\end{lemma}

The constant $K_\eta$ could be found using the critical exponent of a normed space \citep{ptak1962norms}. Let $\rmX$ be an $n$-dimensional normed space and $\|\cdot\|$ be the induced operator norm. Consider a linear operator $\rmA$ on $\rmX$ with norm $1$ and spectral radius $\sigma(\rmA)<1$. It is standard that $\sigma(\rmA)<1$ implies the convergence of $\rmA^n$ to the zero operator as $n\to\infty$, but it may be the case that $\|\rmA\|=1$. The critical exponent gives the smallest possible power $n$ such that $\|\rmA^n\|<1$ for any $\rmA$ with $\sigma(\rmA)<1$.

\cite{ptak1962norms} showed that if $\rmX$ is a Hilbert space, then the critical exponent is equal to $n$. However, the norm $\|\rmQ_\pi \mathbf{1}\|_\infty= \|\rmQ_\pi\|_{\infty\to\infty}$ in question is the operator norm induced by the sup-norm $\|\cdot\|_\infty$ on $\rmX=\gS'$, which has a worse critical exponent of $n^2-n+1$ \citep{mavrik1960norms}. In our setting, it means that there exists $\eta_0\in (0,1)$ such that for any policy $\pi$ we have
\begin{equation*}
    \|(\rmQ_\pi)^{n^2-n+1} \mathbf{1}\|_\infty \leq 1-\eta_0,
\end{equation*}
where $n=|\gS'|$ is the number of transient states. Thus one can take
\begin{equation} \label{K-eta}
    K_\eta =\frac{\log(1-\eta)}{\log(1-\eta_0)}(n^2-n+1). 
\end{equation}

We point out that $K_\eta$ is independent of the number of actions $A$.
While $K_\eta$ is at most $\gO(S^2)$ using the critical exponent, the bound (\ref{K-eta}) could be pessimistic. In many cases $K_\eta$ may even be independent of $S$. For example, if at every state the MDP terminates with some probability $\alpha>0$ under any policy, then we can take $K_{\eta}=\lceil \eta/\alpha \rceil$. As another example, in Blackjack a player can have at most $6$ cards before going bust, so $K_{\eta}\leq 7$ even though $S$ is much larger.


The existence of $K_\eta$ allows us to obtain a tail bound on the episode lengths for any policy, which we state in the following lemma.
\begin{lemma} \label{lemma:subexp-bound-on-absorption-time}
Let $\eta \in (0,1)$. For any policy $\pi$ and state-action pair $(s,a)$, the  random episode length $T^\pi_{s,a}$ satisfies the subexponential bound
\begin{equation*}
\sP \left(T^\pi_{s,a}>t+1\right) \leq C_1 e^{-C_2t} \quad \text{for all}\quad t\geq 0,
\end{equation*}
where
\[ C_1=\frac{1}{1-\eta}, \hspace{5pt} C_2 = \frac{1}{K_\eta} \log \frac{1}{1-\eta}.
\]
\end{lemma}

For an action value function $Q$ we denote by $\Pi(Q)$ the set of all greedy policies with respect to $Q$. With a slight abuse of notation we let $\pi(Q)$ denote any member of $\Pi(Q)$, namely $\pi(Q)$  is any greedy policy with respect to $Q$.

Let $\pi_t$ be the policy at the start of iteration $t$ and $\hat{q}_t$ be the estimate for $q_{\pi_t}$ given by Line 6 
in Algorithm \ref{alg:mces-modified}. The policy improvement step in Line 8 generates the next policy by taking $\pi_{t+1} \in \Pi(\hat{q}_t)$. The policy improvement theorem  \citep{sutton2018reinforcement} states that $\pi(q_{\pi_t})$ improves upon $\pi_{t}$ unless $\pi_{t}$ is already optimal. Thus, if $\pi_{t}$ suboptimal, a sufficient condition for $\pi_{t+1}$ to be an improvement is $\Pi(\hat{q}_t)\subseteq \Pi(q_{\pi_t})$. This condition is satisfied when $\hat{q}_t=q_{\pi_t}$, but this is not generally possible with finitely many samples. Instead, we observe that the condition still holds true when
\begin{equation} \label{q-function-separation}
    |\hat{q}_t(s,a) - q_{\pi_t}(s,a)| < \Delta/2
\end{equation}
for any state-action pair $(s,a)$. To see this, for any $s\in \gS'$ let $a_+ \in \mathcal{A}(\pi_t, s)$ be a maximizing action of the current value function and $a_-$ be an action that corresponds to the second largest value of $q_{\pi_t}(s,\cdot)$. Then for any action $a$ with $q_{\pi_t}(s,a) < q_{\pi_t}(s,a_+)$ we have
\begin{align*}
    \hat{q}_t(s,a) < q_{\pi_t}(s,a) + \Delta/2 & \leq  q_{\pi_t}(s,a_-) + \Delta(\pi_t)/2 \\
    & \leq q_{\pi_t}(s,a_+) - \Delta(\pi_t)/2 \leq  q_{\pi_t}(s,a_+) - \Delta/2 < \hat{q}_t(s,a_+).
\end{align*}
By the law of large numbers, (\ref{q-function-separation}) holds when the number of samples goes to infinity. Lemma \ref{lemma:sampling-error-bound} gives a finite error bound using Hoeffding's inequality and Lemma \ref{lemma:subexp-bound-on-absorption-time}.

\begin{lemma} \label{lemma:sampling-error-bound}
Let $\eta \in (0,1)$. For a policy $\pi$ and state-action pair $(s,a)$ let $\hat{q}(s,a)$ be sample mean of returns obtained from $N$ episodes starting from $(s,a)$ and following $\pi$. Then for any $T_0\in\mathbb{N}$,
\begin{align*}
    \sP  (|\hat{q}(s,a)-q_\pi(s,a)|\geq \Delta/2) \leq 
    2\exp\left(\frac{-\Delta^2 N}{2 T_0^2}\right)  + C_1 e^{-C_2(T_0-1)},
\end{align*}
where $C_1$ and $C_2$ are defined in Lemma \ref{lemma:subexp-bound-on-absorption-time}.
\end{lemma}

Choosing $T_0$ and $N$ sufficiently large, we can ensure that with high probability, each iteration of Algorithm \ref{alg:mces-modified} satisfies (\ref{q-function-separation}) and hence $\pi_{t+1} \in \Pi(q_{\pi_t})$, and since there are only finitely many  (more precisely $S^A$) policies, $\pi_t$ will reach an optimal policy in $S^A$ steps. Thus taking $L=S^A$ in Algorithm \ref{alg:mces-modified} we can an optimal policy. However, each policy evaluation step is expensive, and having to run exponentially many iterations produces trivial theoretical guarantee and makes the algorithm useless. Theorem \ref{theorem:policy-iteration-steps} gives a better upper bound on the number of step sufficient for the exact policy iteration algorithm to reach optimality.

\begin{theorem} \label{theorem:policy-iteration-steps}
    Let $\pi_0$ be any policy and $\{\pi_k\}_{k\geq 0}$ denote the sequence of policies obtained under
    exact policy iteration, namely $\pi_{k+1} \in \Pi(q_{\pi_k})$. 
    Then for any $\eta>0$, the number of policy improvement steps $k$ required before $\pi_k$ reaches an optimal policy is at most 
    $L_0:=\left\lceil 2 \|w\|_\infty \log (\|w\|_\infty/\Delta_\star) \right\rceil$. Namely, $k\geq L_0$ implies $ \pi_k\in\Pi_\star$. 

In particular, since $L(\eta)\geq L_0$ for any $\eta>0$, after running $k\geq L(\eta)$ iteration of exact policy iteration we have $\pi_k\in\Pi_\star$.
\end{theorem}

\begin{proof}
    By Prop 2.5.9 in \cite{bertsekas1995dynamicVol2}, the error of exact policy evaluation can be bounded by
    \begin{align*}
        \|v_{\pi_{k+1}} - v_\star \|_w \leq \rho \|v_{\pi_{k}} - v_\star \|_w \implies 
        \|v_{\pi_{k}} - v_\star \|_w \leq \rho^k \|v_{\pi_{0}} - v_\star \|_w. 
    \end{align*}
    Note that $\pi_k$ is optimal if and only if $\|v_{\pi_{k}} - v_\star \|_w=0$. Given the exponential decay of error, an upper bound on $k$ can be obtained if we can control the minimum distance between $v_\star$ and $v_\pi$ for all the suboptimal policies $\pi$.
    
    Let $\gF$ be the collection of suboptimal policies. Then for any $\pi\in\gF$ and $i\in\gS'$ we have
    \begin{align*}
        v_\star(i)-v_\pi(i) & = q_{\pi^\star}(i,\pi^\star(i)) - q_\pi(i, \pi(i)) \\
        & = [q_{\pi^\star}(i,\pi^\star(i))- q_{\pi^\star}(i,\pi(i)) ] + [q_{\pi^\star}(i,\pi(i)) - q_\pi(i, \pi(i))] \\
        & \geq \Delta_\star \mathds{1}\{\pi(i) \text{ is suboptimal}\} + \sum_{j\in\gS'} p(j | i, \pi(i)) [v_\star(j)-v_\pi(j)] \\
        & = \Delta_{\star}\chi_\pi(i) + [\rmQ_\pi (v_\star-v_\mu)](i),
    \end{align*} 
    where the $i$-th component of $\chi_\pi$ indicates whether $\pi(i)$ does not maximize $q_\star(i, \cdot)$. Thus 
    \begin{align*}
    (\mathbf{I}-\rmQ_\pi) (v_\star-v_\pi) \geq \Delta_{\star} \chi_\pi
    \end{align*} 
    where $\mathbf{I}$ is the identity matrix. Multiplying both sides by $\rmQ_\pi^m$ and summing across $m\geq 0$ gives
    \[v_\star-v_\pi \geq \Delta_{\star} (\mathbf{I}-\rmQ_\pi)^{-1} \chi_\pi. \]
    The inverse is well-defined because we have assumed that all stationary policies are proper. Observe that the $i$-th component of the vector
    \[ 
    w_\pi := (\mathbf{I}-\rmQ_\pi)^{-1} \mathbf{1} \geq \mathbf{1}
    \]
    gives the expected episode length starting from state $i$ following $\pi$, and $\pi\in\gF$ means that $\chi_\pi(i)=1$ for some $i$. Therefore
    \begin{align*}
        \|v_\star-v_\pi\|_w & \geq \frac{\Delta_\star}{\|w\|_\infty} \max_{i=1,...,n} (\mathbf{I}-\rmQ_\pi)^{-1} \chi_\pi \geq \frac{\Delta_\star}{\|w\|_\infty} \min_{i=1,...,n} w_\pi(i) \geq \frac{\Delta_{\star}}{\|w\|_\infty} 
    \end{align*}
 and taking the minimum across suboptimal policies gives
 \[
 \min_{\pi\in\gF} \|v_\star-v_\pi\|_w \geq \frac{\Delta_\star}{\|w\|_\infty}.
 \]

As a result, $\pi_k$ is optimal if $k$ is so large that
 \[ 
 \rho^k \|v_{\pi_0}-v_\star\|_w <\frac{\Delta_\star}{\|w\|_\infty}. 
 \]
 Namely, we can take $k$ to be the smallest integer larger than the following:
 \begin{align*}
   \dfrac{\log \frac{\Delta_{\star}}{\|w\|_\infty \|v_{\pi_0}-v_\star\|_w}}{\log \rho} =  \dfrac{\log \frac{\|w\|_\infty \|v_{\pi_0}-v_\star\|_w}{\Delta_{\star}}}{-\log \left(1-\|w\|_\infty^{-1}\right)} \leq \|w\|_\infty \log \frac{\|w\|_\infty \|v_\star-v_{\pi_0}\|_w}{\Delta_{\star}}.   
 \end{align*}
 Since we assumed that the immediate reward is bounded in $[0,1]$, we have
 \[
 \|v_\star-v_{\pi_0}\|_w \leq \|v_\star-v_{\pi_0}\|_\infty \leq \|v_\star\|_\infty \leq \|w\|_\infty
 \]
 and therefore exact policy iteration converges after at most 
 $
    \left\lceil 2 \|w\|_\infty \log (\|w\|_\infty/\Delta_\star) \right\rceil
 $
 iterations.

 Finally, to show that $L(\eta)\geq L_0$ it sufficies to show that $\|w\|_\infty\leq K_\eta/\eta$.
 Let $\mu$ be a policy that attains $w$ (i.e. one that maximizes the expected episode length for every state). Denote by $\|\cdot\|$ the matrix norm induced by $\|\cdot\|_\infty$. Then by (\ref{norm-decrease}) we have
 \begin{align*}
     \|w\|_\infty & = \|(\mathbf{I}-\rmQ_\mu)^{-1}\mathbf{1}\|_\infty = \|(\mathbf{I}-\rmQ_\mu)^{-1}\| \\
     & =\left\Vert\sum_{k=0}^\infty \rmQ_\mu^k\right\Vert \leq \sum_{k=0}^\infty \left\Vert \rmQ_\mu^k\right\Vert \leq \sum_{m=0}^\infty K_\eta \| \rmQ_\mu^{mK_\eta}\|  \\
     & \leq \sum_{m=0}^\infty K_\eta \| \rmQ_\mu^{K_\eta}\|^m \leq \sum_{m=0}^\infty K_\eta (1-\eta)^m = K_\eta/\eta.
 \end{align*}
\end{proof}

We are now ready to prove Theorem \ref{theorem:main}, which is a direct consequence of Lemma \ref{lemma:sampling-error-bound} and Theorem \ref{theorem:policy-iteration-steps}.

\begin{proof}[Proof of Theorem \ref{theorem:main}]
    Set  $\eta:=1-e^{-1}$, so that $C_1=e$ and $C_2=K^{-1}$.
    For $\delta \in (0,1)$ let $\zeta:=1-(1-\delta)^{1/L_\star}$. If at each policy evaluation step, the probability of correctly identifying the maximizing action for all states is at least $1-\zeta$, then by Lemma \ref{theorem:policy-iteration-steps}, the probability that Algorithm \ref{alg:mces-modified} returns an optimal policy with $L=L_\star$ is at least $(1-\zeta)^{L_\star}=1-\delta$. 
    
    In light of Lemma \ref{lemma:sampling-error-bound} and union bound over $\gS'\times \gA$, it suffices to choose $T_0$ and $N$ so that
    \begin{align} \label{separate-bound}
        C_1 e^{-C_2 (T_0-1)} \leq \frac{\zeta}{2SA}
         \hspace{5pt}\text{ and }\hspace{5pt} 
         2\exp\left(\frac{-\Delta^2 N}{2 T_0^2}\right) \leq  \frac{\zeta}{2SA}.
    \end{align}
    By setting
    \begin{equation*}
        T_0 := 2K \log \frac{4SA}{\zeta},
    \end{equation*}
    we have
    \begin{align*}
        C_1 e^{-C_2 (T_0-1)} & = e^{1-\frac{T_0-1}{K}} = e^{1+\frac{1}{K}} \left(\frac{\zeta}{4SA}\right)^2 \leq \left(\frac{e\zeta}{4SA}\right)^2 \leq \frac{\zeta}{2SA}.
    \end{align*}
    and the first inequality is satisfied. Solving for the second inequality we get
    \begin{align*}
        2\exp\left(\frac{-\Delta^2 N}{2 T_0^2}\right) \leq \frac{\zeta}{2SA} \implies N \geq \frac{2 T_0^2}{\Delta^2} \log \frac{4SA}{\zeta} = \frac{8K^2}{\Delta^2}  \left(\log \frac{4SA}{\zeta}\right)^3,
    \end{align*}
    which gives the expression (\ref{n-num-samples}) for $N(\delta)$.
\end{proof}
\begin{proof}[Proof of Corollary \ref{corollary:main}]
Expanding $(1-\delta)^{1/L_\star} = 1-\delta/L_\star + \gO(\delta^2/L_\star^2)$ gives 
    \begin{align*}
        \log \frac{1}{1-(1-\delta)^{\frac{1}{L_\star}}} = \log\frac{1}{\delta/L_\star + \gO(\delta^2/L_\star^2)} = \gO \left( \log \frac{L_\star}{\delta} \right),
    \end{align*}
    from which the sample complexity bound follows.
\end{proof}

\subsection{Comparison with Existing Results and the Finite-horizon Setting}
We point out some commonality and differences between our result and the sample complexity bound  by  \cite{chen2023reaching}.
First, our formulation is a maximization problem and the one in \cite{chen2023reaching} is minimization. The condition that the immediate cost $c(s,a)$ and the immediate reward $r(s,a)$ is bounded in $[0,1]$ makes the conversion between the two nontrivial (without the range restriction, one can simply set $r=-c$). Structurally speaking, they assumed that the MDP has a special terminating action $a_\dagger$ which terminates the episode from any state with a fixed common cost $J$, and the minimum immediate cost $c_{\min}$ over all state-action pairs is strictly positive. In contrast, we do not impose assumptions on the immediate reward or the existence of $a_\dagger$ but require that all policies are proper (and that $\Delta$, $\Delta_\star$ and $K_\eta$ are known). These two sets of assumptions on the environment are not comparable; namely, neither set of assumptions implies the other. 

In addition, the two sample complexity bounds are not directly comparable either. The bound by \cite{chen2023reaching} for learning an $\epsilon$-optimal policy is of order
\begin{equation} \label{eq:bound-previous}
    \tilde{\gO}\left(\frac{SA \|v_\star\|_\infty^3}{c_{\min}\epsilon^2}+\frac{S^2A^2 J^4 \|v_\star\|_\infty}{c_{\min}^3 \epsilon} \right),
\end{equation}
while Theorem \ref{theorem:main} gives $\tilde{\gO}(SA K^3/ \Delta^2)$, where $\Delta$ plays a similar role as $\epsilon$ as they both indicate the accuracy of the value function estimates.
Which term of the two in (\ref{eq:bound-previous}) is dominating is instance-specific.
If an MDP has small numbers of states and actions but each policy induces a graph with high probability of looping back at each state, then $\|v_\star\|_\infty$ becomes large and hence the first term is dominant. In this case it gives better bound than Theorem \ref{theorem:main} because $\|v_\star\|\leq \|w\|_\infty\leq K_\eta/\eta$. However, for environments in which the product $SA$ is significantly larger than $\|v_\star\|_\infty$ (such as Chess), the second term in (\ref{eq:bound-previous}) becomes dominant and hence Theorem \ref{theorem:main} gives a better bound.

Finally, we point out that Theorems $\ref{theorem:main}$ and \ref{theorem:policy-iteration-steps} easily specialize to the finite-horizon setting. If $H$ is the horizon length, we may choose $K_\eta=H$ for any $\eta$ and notice that $\|w\|_\infty = H$. Thus Theorem \ref{theorem:policy-iteration-steps} says that exact policy iteration takes at most $\lceil 2H \log(H/\Delta_\star) \rceil$ to reach an optimal policy. It is also easy to see that the complexity in Corollary \ref{corollary:main} can actually be taken to be $SA L_0 N'(\delta)$ where $N'(\delta)$ differs from $N(\delta)$ only by replacing $L_\star$ with $L_0$. As such, the overall complexity is of order $\tilde{O}(SAH^3 \log^3(1/\delta))$. This is close to the optimal minimax sample complexity  $\tilde{O}(SAH^3\epsilon^{-2} \log(1/\delta))$ \citep{menard2021fast}, and we believe that the dependency of our bound on $\delta$ can be further improved if more sophisticated sampling and stopping strategies are used.

\section{Conclusion}
In this paper we give a finite sample analysis on a Monte Carlo algorithm for episodic undiscounted MDPs, namely the stochastic shortest path problem. The main challenge of analyzing SSPPs is the fact that episode length can be unbounded, which can make the error in statistical estimation arbitrarily large. In order to bound this error, we directly exploit the subexponential behavior of the episode length by leveraging the combinatorial structure of matrices with infinity norm. In addition, we quantify the policy improvement theorem using the underlying contraction structure and hence upper bound the number of policy improvement steps needed to reach an optimal policy.

Limitations of our analysis include the assumptions that all policies are proper and that  $K_\eta$, $\Delta$ and $\Delta_\star$ are assumed to be known in order to achieve the desired statistical accuracy. In many tasks, it is not unreasonable to assume the properness of all policies as well as prior knowledge of $K_\eta$, such as Blackjack. However, it is more restrictive to assume prior knowledge on the suboptimality gaps $\Delta$ and $\Delta_\star$, as they depend not only  on the immediate reward $r$ but also on the transition structure $p$. Nevertheless, our sample complexity bound is competitive because it near optimal dependence on the size of the MDP and proxy horizon length.  In particular, with knowledge of these parameters we can recover \emph{exact} optimal instead of just $\epsilon$-optimal policies with an explicit number of trajectories.

We identify several possible future research directions on studying sample complexity bounds of MCES. One is to remove the assumption that all policies are proper by setting an upper bound on the length of trajectories being sampled \citep{chen2023reaching}.
Another is to relax the synchronocity assumption with some form of exploration strategy such as upper confidence bounds \citep{dong2022convergence} or best-arm identification \citep{de2021bandits}, even though the complete removal of sampling assumptions might not be possible. Our analysis on exact policy iteration also allows one to work with more elaborate sampling schemes than the one considered in the paper, which can lead to improved sample complexity.

\clearpage
\bibliography{mces}
\bibliographystyle{plainnat}
\clearpage

\appendix
\section{Proofs of Lemmas} \label{appendix:proofs}
\subsection{Proof of Lemma \ref{lemma:uniform-bound-on-transition-matrix}}
\label{appendix-lemma1}
\begin{proof}
    Fix $\eta\in (0,1)$. For any policy $\pi$ with transition matrix $\rmQ_\pi$ on $\gS'$,  we have  \begin{equation}\|\rmQ_\pi^{k+1}\mathbf{1}\|_\infty \leq \|\rmQ_\pi^{k}\mathbf{1}\|_\infty \|\rmQ_\pi\mathbf{1}\|_\infty  \leq \|\rmQ_\pi^{k}\mathbf{1}\|_\infty \label{norm-decrease}
    \end{equation}
    for $k\in\mathbb{N}$, so the expression $\|\rmQ_\pi^{k}\mathbf{1}\|_\infty$ is decreasing in $k$. Since $\pi$ is proper, $\rmQ_\pi^{k}$ converges to the zero matrix as $k\to\infty$, so there exists $k(\pi)\in \mathbb{N}$ such that $\|\rmQ_\pi^{k}\mathbf{1}\|_\infty\leq 1-\eta$ for all $k\geq k(\pi)$. Since there are finitely many policies, it suffices to take $K_\eta:=\max_{\pi} k(\pi)<\infty$.
\end{proof}

\subsection{Proof of Lemma \ref{lemma:subexp-bound-on-absorption-time}}
\begin{proof}
Fix any policy $\pi$ and state-action pair $(s,a)$. Recall that $T^\pi_{s,a}$ denote the length of a random trajectory 
\[
S_0 = s, A_0 = a, S_1, \dots,  S_{T^\pi_{s,a}}=\tilde{s}
\]
starting from $(s,a)$ and following $\pi$.
For any distribution $\tau$ of $S_1$ on $\gS'$ let $T_\tau= T^\pi_{s,a}-1$. Then
\begin{align*}
\sP (T_\tau > k \mid T^\pi_{s,a}>1) & = \tau \rmQ_\pi^k \mathbf{1} , \hspace{10pt} k\geq 0.
\end{align*}
Let  $p_{s,a}:=1-p(\tilde{s}\mid s,a)$. Thus $\sP (T^\pi_{s,a}=1)=1-p_{s,a}$ and for $k\geq 0$,
\begin{align*}
\sP (T^\pi_{s,a}>k+1) & = \sP (T^\pi_{s,a}>1) \sP (T^\pi_{s,a}>k+1 \mid T^\pi_{s,a}>1)  = \sP (T^\pi_{s,a}>1) \sP (T_\tau>k \mid T^\pi_{s,a}>1) \\
& = p_{s,a} \tau_{s,a} \rmQ^{k}\mathbf{1} \leq p_{s,a} \|\rmQ^k\mathbf{1}\|_\infty,
\end{align*}
where $S_1\sim\tau_{s,a}$ is the distribution of $S_1$ on $\gS'$. Thus for $t\geq 0$ we have
\begin{align*}
& \sP (T_{s,a} > t+1)\leq \sP (T_{s,a} > \lfloor t/K_\eta\rfloor K_\eta + 1) \leq p_{s,a}(1-\eta)^{\lfloor t/K_\eta\rfloor} \leq C_1 e^{-C_2t},
\end{align*}
where $C_1=\frac{1}{1-\eta}$ and $C_2=\frac{1}{K_\eta}\log(\frac{1}{1-\eta})$.
\end{proof}

\subsection{Proof of Lemma \ref{lemma:sampling-error-bound}}
\begin{proof}
    Let $E$ be the event that $|\hat{q}(s,a)-q_\pi(s,a)| \geq \Delta/2$. Then for $T_0\in\mathbb{N}$, 
    \begin{align*}
        \sP (E) & = \sP (E \cap \{T^\pi_{s,a}\leq T_0\}) + \sP (E \cap \{T^\pi_{s,a}> T_0\}) \\
        & \leq \sP (E \cap \{T^\pi_{s,a}\leq T_0\}) + \sP (T^\pi_{s,a}> T_0) \\
        & \leq \sP (E \cap \{T^\pi_{s,a}\leq T_0\}) + C_1 e^{-C_2 (T_0-1)}.
    \end{align*}
    Since in the event $\{T^\pi_{s,a}\leq T_0\}$, $\hat{q}(s,a)$ is the average of $N$ independent random variables bounded in $[0, T_0]$, the first term can be controlled using Hoeffding's inequality
    \begin{align*}
        \sP (E \cap \{T^\pi_{s,a}\leq T_0\}) \leq 2\exp\left(\frac{-\Delta^2 N}{2 T_0^2}\right),
    \end{align*}
    which completes the proof.
\end{proof}

\end{document}